\documentclass[sigconf,nonacm]{acmart}
\AtBeginDocument{%
  }

\setcopyright{acmlicensed}
\copyrightyear{2018}
\acmYear{2018}
\acmDOI{XXXXXXX.XXXXXXX}
\acmConference[Conference acronym 'XX]{Make sure to enter the correct
  conference title from your rights confirmation email}{June 03--05,
  2018}{Woodstock, NY}
\acmISBN{978-1-4503-XXXX-X/2018/06}

\usepackage{algorithm}
\usepackage{algorithmic} 
\usepackage{booktabs}
\usepackage{multirow}
\usepackage{graphicx}
\usepackage{subcaption}
\begin{document}

\title{ShiftLIF: Efficient Multi-Level Spiking Neurons with Power-of-Two Quantization}


% Kaiwen Tang, Di Yu, Jiaqi Zheng, Changze Lv, Qianhui Liu, Zhanglu Yan, Weng-Fai Wong

\author{Kaiwen Tang}
\authornote{Both authors contributed equally to this research.}
% \email{tang_kaiwen@u.nus.edu}
\affiliation{%
  \institution{National University of Singapore}
  \country{}
}

\author{Di Yu}
\authornotemark[1]
\affiliation{%
  \institution{Zhejiang University}
  \country{}
}

\author{Jiaqi Zheng}
\affiliation{%
  \institution{Sea AI Lab}
  \country{}
}

\author{Changze Lv}
\affiliation{%
 \institution{Fudan University}
 \country{}
}

\author{Qianhui Liu}
\affiliation{%
  \institution{Shandong University}
  \country{}
}

\author{Zhanglu Yan}
\authornote{Corresponding author.}
\affiliation{%
  \institution{National University of Singapore}
  \country{}
}

\author{Weng-Fai Wong}
\affiliation{%
  \institution{National University of Singapore}
  \country{}
}

% \affiliation{%

\renewcommand{\shortauthors}{Tang et al.}


\begin{abstract}
Spiking neural networks (SNNs) are promising for edge sensing due to their event-driven computation and temporal filtering capability. However, standard leaky integrate-and-fire (LIF) neurons communicate only through binary spikes, which severely limit representational capacity. Existing multi-level spiking neurons improve information transmission, but often rely on uniform quantization that mismatches membrane-potential distributions or introduces costly synaptic multiplications.
In this paper, we propose \textit{ShiftLIF}, a multi-level spiking neuron that maps membrane potentials to a logarithmically spaced power-of-two spike set. This design provides finer representation in the small-amplitude regime, where membrane potentials are densely concentrated, while enabling multiplier-free synaptic computation through bit-shift and accumulation operations. As a result, ShiftLIF improves spike-level expressiveness without sacrificing the hardware-friendly nature of standard SNN computation.
We evaluate ShiftLIF on 10 datasets spanning wireless, acoustic, motion, and visual sensing tasks. Results show that ShiftLIF consistently matches or exceeds the accuracy of existing multi-level spiking neurons while maintaining synaptic energy consumption close to standard binary LIF. These results indicate that ShiftLIF provides a favorable accuracy-efficiency trade-off for cross-modal edge sensing.
\end{abstract}





\maketitle

\section{Introduction}
% <!-- Background: edge multimodal sensing needs energy efficiency -->
The rapid expansion of edge computing has driven an increasing demand for sensing applications across wireless, acoustic, and motion domains~\cite{lai2023ai, shi2016edge, wang2023edge}. However, processing these continuous streams on edge devices poses a challenge to energy consumption~\cite{merluzzi2021discontinuous}.
{\em Spiking neural networks} (SNN) have emerged as a highly promising paradigm to address this~\cite{zhang2024spiking, zhou2023computational, tang2026spikyspace}. The most widely adopted mechanism in these networks is the {\em leaky integrate-and-fire} (LIF) neuron~\cite{teeter2018generalized}. Beyond providing event-driven energy efficiency, the leaky integration dynamics intrinsically act as a temporal low-pass filter~\cite{fangspiking, connelly2016thalamus}. This enables the neuron to smooth out high-frequency environmental noise and effectively 
capture slowly varying temporal patterns in real-world physical signals.

% <!-- Strength: Dilemma: Loses info -->
Despite this advantage, standard LIF neurons suffer from a severe representational bottleneck~\cite{xiao2024multi}. Fundamentally, spiking neurons compute rich analog membrane potentials but encode them with only one bit. 
Because neurons communicate strictly through binary spikes, this extreme quantization discards much of the amplitude information preserved by temporal integration~\cite{han2020rmp, luo2024integer}. As a result, the expressive capacity of the network is significantly constrained, limiting its ability to capture the full complexity of continuous sensory signals.

% <!-- Previous Methods and shortcomings -->
Recent efforts to address this limitation generally follow two directions. The first line of work enhances neuron expressiveness by modifying membrane dynamics, such as PLIF~\cite{fang2021incorporating}, GLIF~\cite{yao2022glif}, CLIF~\cite{huang2024clif}, and RPLIF~\cite{li2025incorporating}, which introduce learnable time constants or additional state variables. While these approaches enrich temporal modeling, they increase computational overhead and still constrain neuron outputs to binary spikes, leaving the quantization bottleneck unresolved. The second line of work increases representational bandwidth by introducing multi-level spikes, such as LM-HT SNN~\cite{hao2024lm} and INT-LIF~\cite{luo2024integer}. However, these methods typically rely on uniformly spaced spike levels, which introduce two limitations: multi-bit spike amplitudes require synaptic multiplications, increasing hardware cost, and uniform quantization poorly matches the intrinsic distribution of membrane potentials, which are often concentrated near small values with a long tail toward larger ones. Consequently, uniform grids allocate excessive resolution to rare large activations while under-representing densely clustered small activations.

% <!-- Proposed method and advantage -->
To address these challenges, we propose ShiftLIF, a novel spiking neuron that both improves representational capacity and hardware efficiency. ShiftLIF maps the continuous membrane potential to a logarithmically spaced set of spikes defined by powers of two, $\{0, 2^{-K}, 2^{-(K-1)}, \dots, 2^{-1}, 1\}$.
This design provides two key advantages. 
Algorithmically, the power-of-two grid is naturally denser near zero, which better matches the heavy-tailed distribution of membrane potentials and captures fine-grained variations in small activations that uniform quantization often misses~\cite{guo2023rmp, przewlocka2022power}. From a hardware perspective, because spike amplitudes are restricted to powers of two and synaptic weights are represented as integers, spike–weight interactions can be implemented using simple bit-shift operations followed by integer accumulation~\cite{ tann2017hardware, li2019additive}. Consequently, ShiftLIF enables fully multiplier-free synaptic computation, preserving the core energy efficiency advantage of spiking neural networks. Moreover, since the bit-shift operation effectively reduces the bit-width of the shifted weights, the subsequent accumulation involves smaller operands, which can further reduce switching activity and energy consumption compared to standard accumulation~\cite{solanki2025atm, coward2024combining}.

To evaluate ShiftLIF, we conduct extensive experiments across 11 diverse datasets encompassing wireless, acoustic, motion, and visual sensing modalities. Empirical results demonstrate that ShiftLIF consistently achieves competitive or superior classification accuracy compared to state-of-the-art multi-level spiking neurons. Furthermore, it accomplishes this while consuming lower synaptic energy comparable to that of standard binary LIF models. % thereby establishing an optimal accuracy-efficiency trade-off largely unattained by existing methodologies.

The main contributions of this paper are:
\begin{enumerate}
\item We identify the fundamental accuracy–efficiency dilemma in existing multi-level SNNs, where uniform quantization increases computational overhead and poorly matches the natural distribution of membrane potentials.

\item We propose ShiftLIF, a novel spiking neuron that generates multi-level spikes constrained to powers of two, enabling distribution-aware representation while preserving multiplier-free synaptic computation.

\item We establish a comprehensive cross-modal benchmark across wireless, acoustic, motion, and visual sensing tasks, demonstrating that ShiftLIF achieves state-of-the-art accuracy while significantly improving energy efficiency over other LIF variants.
\end{enumerate}

\begin{figure*}[t] 
    \centering
    \includegraphics[width=\textwidth]{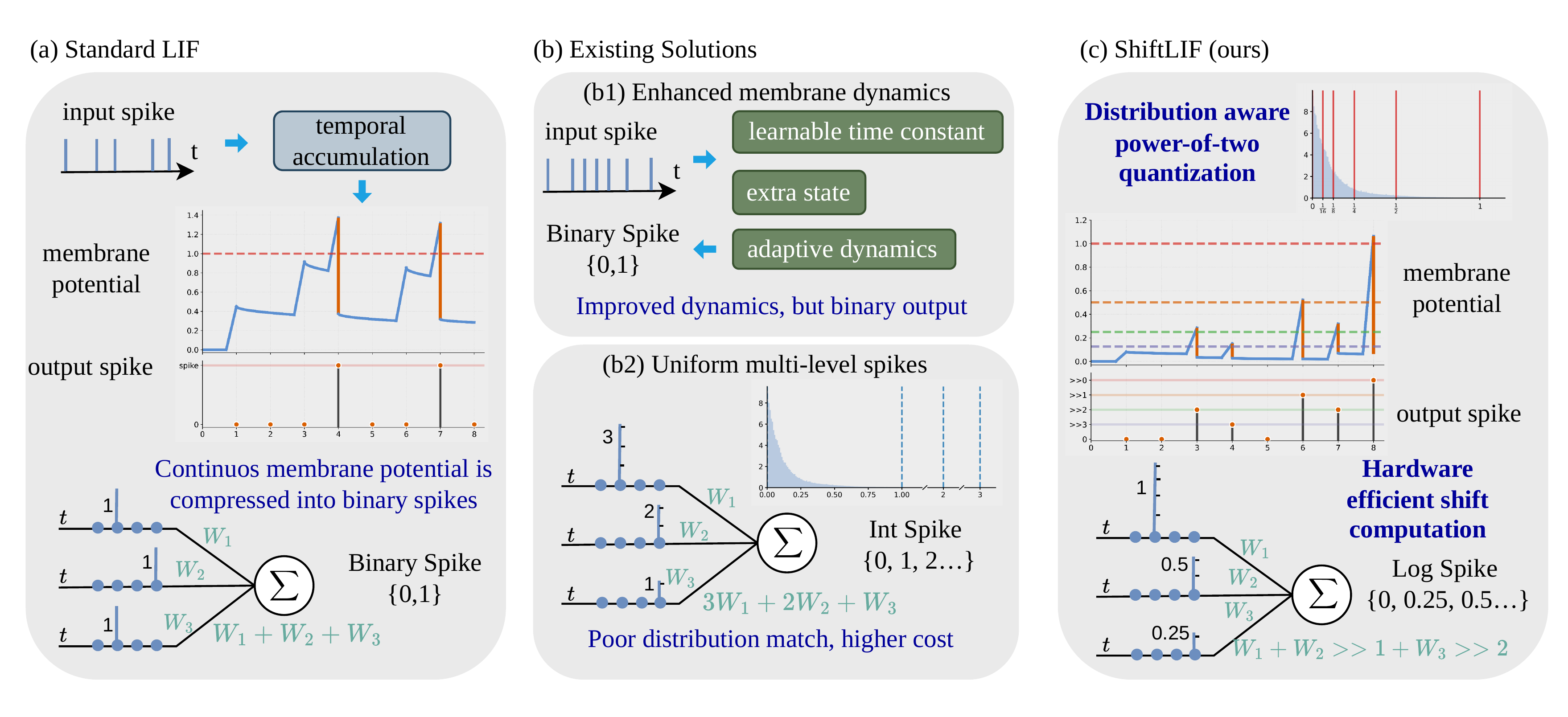}
    \caption{Comparison of standard LIF, existing solutions, and ShiftLIF. In (c), the colored dashed lines mark the thresholds of different power-of-two output levels, and the top inset shows a small-value-dominant membrane distribution with quantization points denser near zero. At the bottom, $W \gg k$ denotes a bitwise right shift by $k$ bits, highlighting that ShiftLIF replaces multiplication with shift-and-accumulate computation.
}
    \label{fig:wide_image}
\end{figure*}

\section{Related Work}

\subsection{SNNs for Edge Sensing Data}

Spiking Neural Networks have been increasingly explored for edge sensing applications, including acoustic recognition, wearable motion analysis, event-based vision, and wireless sensing, due to their event-driven sparsity and native temporal processing capability~\cite{deng2025edge,sabbella2025promise,baek2024snn, li2023efficient}. Existing sensing-oriented SNN studies have primarily emphasized architectural adaptation, temporal feature extraction, encoding strategies, and surrogate gradient for direct training~\cite{neftci2019surrogate, guo2023direct, malcolm2023comprehensive}. In comparison, the spike representation mechanism itself, especially the design of the neuron output alphabet beyond binary communication, remains relatively underexplored in SNNs for sensing applications design.

\subsection{LIF Neurons and Their Variants}
To overcome the representational limitations of standard LIF dynamics, numerous variants have been proposed, broadly falling into two categories.
The first category enriches internal membrane dynamics while strictly retaining binary spike communication. For instance, PLIF \cite{fang2021incorporating} introduces learnable membrane time constants, GLIF \cite{yao2022glif} incorporates complex gating mechanisms, and CLIF \cite{huang2024clif} adds complementary paths for better gradient flow. While these variants significantly improve temporal modeling and optimization behavior, they ultimately compress the rich, continuous membrane state into a 1-bit binary spike, leaving the inter-layer communication bottleneck unresolved.

The second category directly expands the spike alphabet to mitigate this 1-bit limitation. Representative works include multi-threshold or ternary neurons, such as Ternary Spike \cite{guo2024ternary}, MT-SNN \cite{wang2023mt}, I-LIF \cite{sun2025ilif}, and INT-LIF~\cite{luo2024integer}. While these multi-level neurons increase the information carried per spike, they typically rely on uniformly spaced quantization grids or generic multi-bit integer amplitudes. This introduces two critical flaws: first, as demonstrated by recent studies \cite{guo2022recdis, guo2023rmp}, uniform discretization severely mismatches the zero-concentrated, heavy-tailed distribution of membrane potentials, leading to suboptimal level utilization and high relative quantization error~\cite{li2019additive}. Second, utilizing generic multi-bit spikes inevitably reintroduces dense Multiply-Accumulate (MAC) operations at the synapses, fundamentally destroying the hardware-efficiency premise of SNNs. 

ShiftLIF structurally advances this second category. By introducing a logarithmic, power-of-two quantization grid, it simultaneously achieves distribution-aware information maximization and strictly MAC-free, shift-based synaptic computation.

\section{Method}
\subsection{Preliminary: Standard Leaky IF}
The leaky integrate-and-fire (LIF) neuron is the most widely adopted computational unit in spiking neural networks~\cite{voudaskas2025spiking, neftci2019surrogate}. At each discrete time step $t$, the neuron integrates the incoming synaptic current into its membrane potential while gradually decaying previous states. Incorporating a soft-reset mechanism, the membrane dynamics can be formulated as
\begin{equation}
u_t = \lambda u_{t-1} + x_t - s_{t-1}V_{\mathrm{th}},
\end{equation}
where $u_t$ denotes the membrane potential, $\lambda \in [0,1)$ is the leak factor, $x_t$ is the input current, $V_{\mathrm{th}}$ is the firing threshold, and $s_{t-1}$ represents the spike emitted at the previous time step. When the accumulated membrane potential exceeds the threshold, the neuron generates an output spike according to a Heaviside step function:
\begin{equation}
s_t =
\begin{cases}
1, & \text{if } u_t \geq V_{\mathrm{th}},\\
0, & \text{otherwise}.
\end{cases}
\end{equation}

From a signal processing perspective, the leakage term $\lambda$ endows the LIF neuron with the behavior of a first-order infinite impulse response low-pass filter~\cite{fangspiking}. It enables the neuron to temporally integrate recent inputs, effectively suppressing high-frequency environmental noise while preserving slowly varying signal components that are important for real-world sensory perception.

The binary firing mechanism also makes synaptic computation highly efficient, since the interactions between spikes and weights are reduced to simple accumulation without multiplications. However, the LIF neuron suffers from a representational bottleneck because its continuous membrane state is communicated only through binary spikes \(s_t \in \{0,1\}\)~\cite{guo2023rmp, guo2024ternary}. This extreme quantization discards fine-grained amplitude information accumulated through temporal integration, thereby limiting the expressive capacity of the neuron, especially in continuous sensing scenarios where subtle signal variations are often informative.

To overcome this limitation, the next section introduces ShiftLIF, which replaces the binary spike output of conventional LIF neurons with a power-of-two multi-level spike representation while preserving hardware-friendly computation.

% This raises a key challenge: how can one increase spike-level representational capacity without sacrificing the hardware efficiency of binary synaptic computation? The next section addresses this question by introducing ShiftLIF, which replaces the binary spike output of conventional LIF neurons with a power-of-two multi-level spike representation while preserving hardware-friendly computation.

\subsection{ShiftLIF}
The design principle of ShiftLIF is to \textbf{break the binarized representation bottleneck while preserving the hardware efficiency of standard LIF neurons}. To this end, ShiftLIF retains the membrane integration dynamics of the conventional LIF neuron and modifies only the spike generation and reset processes. 

\subsubsection{Shift-Quantized Activation}
The central challenge in reformulating spike generation is to increase spike-level representational bandwidth without sacrificing hardware efficiency. Existing multi-level spiking neurons typically rely on uniformly spaced quantization levels. While such designs improve amplitude representation, they generally require synaptic multiplications and therefore do not naturally preserve the multiplier-free nature of standard LIF computation. In addition, uniform quantization allocates equal resolution across the full dynamic range, which is not well matched to the non-uniform distribution of membrane potentials that typically cluster near small values~\cite{li2019additive,przewlocka2022power,guo2022recdis,guo2023rmp}.

To address both issues, ShiftLIF replaces the standard Heaviside step function with a shift-quantization operator \(Q_{\text{shift}}\), which maps the continuous membrane potential to a discrete power-of-two spike set. This design naturally produces a logarithmically spaced output space, providing finer resolution for small activations while preserving shift-friendly computation. Without loss of generality, we normalize the firing threshold to \(V_{\mathrm{th}} = 1.0\). The quantization operator is defined as
\begin{equation}
S_t = Q_{\text{shift}}(V_t), \quad
Q_{\text{shift}}(v) =
\begin{cases}
0, & \text{if } v < 2^{-(K+1)}, \\[6pt]
2^{-\lceil -\log_2 v \rceil}, & \text{if } 2^{-(K+1)} \leq v \leq 1,
\end{cases}
\end{equation}
where \(v = \max(0, \min(V_t, 1))\) denotes the bounded membrane potential, and \(K \in \mathbb{N}\) controls the smallest non-zero spike magnitude. The corresponding output alphabet is
\begin{equation}
\mathcal{S} = \{0,\; 2^{-K},\; 2^{-(K-1)},\; \dots,\; 2^{-1},\; 1\}, 
\qquad |\mathcal{S}| = K + 2.
\end{equation}

For example, setting \(K=7\) yields 9 discrete spike levels spanning multiple orders of magnitude, enabling substantially richer amplitude encoding than standard binary spikes.

\subsubsection{Logarithmic Quantization}
The quantization grid of \(Q_{\text{shift}}\) is logarithmically spaced as the bin corresponding to the output \(2^{-k}\), defined on the interval \([2^{-(k+1)}, 2^{-k})\), has width \(2^{-(k+1)}\). This width shrinks exponentially as \(k\) increases, i.e., as activation values approach zero. In contrast, the uniform grid used in INT-LIF assigns a constant bin width \(\Delta = 1/N\) across the full dynamic range.

This non-uniform spacing is motivated by the membrane-potential distribution in SNNs. Due to the continuous leakage dynamics, membrane potentials are typically concentrated near small values rather than being uniformly distributed over the full range. Prior studies on activation quantization and recent analyses of deep SNNs have shown that such non-uniform distributions are closely related to quantization error ~\cite{li2019additive,przewlocka2022power,guo2022recdis,guo2023rmp}. As further illustrated by our empirical measurements across multiple sensing modalities in Figure~\ref{fig:quant_analysis}, the membrane potentials in trained models consistently cluster near zero.

For such non-uniform distributions, uniformly spaced quantization levels allocate equal resolution to both dense and sparse regions, which can under-represent the small-amplitude regime where membrane potentials occur most frequently. By contrast, the logarithmic power-of-two grid in ShiftLIF provides finer resolution near zero while maintaining coarser spacing for rare large values. This makes the quantization grid better aligned with the empirical membrane-potential distribution and naturally supports shift-based computation. Section~\ref{sec:theory} further analyzes the resulting representational properties.

After spike generation, the membrane potential must be reset. To retain the residual sub-threshold charge that is not encoded into the current spike, ShiftLIF employs a proportional soft-reset mechanism:
$$
V_t \leftarrow V_t - S_t \cdot V_{\text{th}}.
$$
Compared with the standard LIF soft reset, which subtracts a fixed threshold after each spike, ShiftLIF decrements the membrane potential proportionally to the emitted multi-level spike magnitude \(S_t \in \mathcal{S}\). This preserves residual membrane charge across time steps and reduces information loss caused by hard state truncation.

The complete forward dynamics of a ShiftLIF neuron at a single time step are summarized in Algorithm~\ref{algo:1}. 

\begin{algorithm}[H]
\caption{ShiftLIF Forward Pass}
\label{algo:1}
\textbf{Input:} Current input $X_t$ at timestep $t$, membrane potential $V_{t-1}$ \\
\textbf{Hyperparameters:} Time constant $\tau$, threshold $V_{\text{th}}=1.0$, resting potential $V_{\text{reset}}=0.0$, precision factor $K$
\begin{algorithmic}[1]
\STATE \textbf{1. Leaky Integration}
\STATE $V_t \leftarrow V_{t-1} + \frac{1}{\tau} (V_{\text{reset}} - V_{t-1}) + X_t$
\STATE \textbf{2. Bounding (Clamp)}
\STATE $v \leftarrow \max(0, \min(V_t, V_{\text{th}}))$
\STATE \textbf{3. Shift-Quantized Activation}
\IF{$v < 2^{-(K+1)}$}
    \STATE $S_t \leftarrow 0$
\ELSE
    \STATE $k \leftarrow \lceil -\log_2 v \rceil$
    \STATE $S_t \leftarrow 2^{-\min(k, K)}$
\ENDIF
\STATE \textbf{4. Proportional Soft Reset}
\STATE $V_t \leftarrow V_t - S_t \cdot V_{\text{th}}$
\STATE \textbf{Return} $S_t, V_t$
\end{algorithmic}
\end{algorithm}

For continuous sequence processing, this forward pass is iteratively applied over $t = 1, \dots, T$, updating the membrane state and accumulating the output spike train $\{S_t\}_{t=1}^T$.

\begin{figure*}[t]
    \centering
    \includegraphics[width=\textwidth]{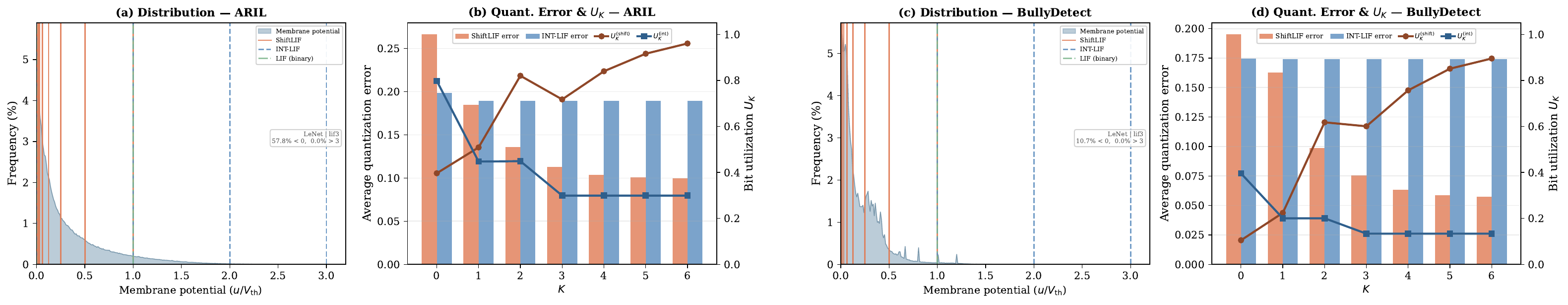}
    \caption{
    Membrane potential distribution, quantization error, and bit utilization metric on sensing datasets. 
    % Standard LIF membrane potentials are concentrated in the low-amplitude regime. 
    Compared with INT-LIF, ShiftLIF yields lower quantization error and better bit utilization level as $K$ increases.
    }
    \label{fig:quant_analysis}
\end{figure*}

\subsection{Theoretical Analysis}
\label{sec:theory}
% The design of ShiftLIF raises two central theoretical questions: \textbf{how multi-level spikes improve communication capacity, and how logarithmic power-of-two quantization maintains precision for small membrane values.} In this section, we address these questions through an analysis of information-theoretic capacity and relative quantization error.

We next compare ShiftLIF with INT-LIF~\cite{luo2024integer}, a representative multi-level spiking baseline based on uniformly spaced integer levels. Standard binary LIF is recovered as the special case \(K=0\) of INT-LIF, so the comparison with LIF is already covered by this formulation and is not discussed separately here. Let \(X\) denote a nonnegative random variable representing the membrane potential normalized by the firing threshold, and let \(v\) denote a generic realization of \(X\). We use \(X\) in distribution-level expressions such as probabilities and expectations, and \(v\) in pointwise definitions of the quantizers. Let \(K\) be the precision factor introduced in Section~3.2.1. To match the same output budget, we compare the two level sets
\begin{equation}
\begin{aligned}
\mathcal{S}_{\text{shift}} &= \{0, 2^{-K}, 2^{-(K-1)}, \dots, 1\},\\
\mathcal{S}_{\text{int}} &= \{0,1,2,\dots,K+1\},
\end{aligned}
\end{equation}
which both contain \(K+2\) discrete outputs. INT-LIF, and hence standard LIF as its \(K=0\) special case, follows the nearest-level rule on a uniform integer grid. ShiftLIF instead selects the largest admissible dyadic level not exceeding the membrane value; this preserves power-of-two spike magnitudes and can be implemented efficiently with bit-shift operations, making it more hardware-friendly. Accordingly, the ShiftLIF quantizer can be interpreted as
\begin{equation}
Q_{\text{shift}}(v)=\max\{s\in\mathcal{S}_{\text{shift}}:s\le v\}.
\end{equation}
The INT-LIF quantizer is defined by
\begin{equation}
Q_{\text{int}}(v)=
\begin{cases}
0, & 0\le v<\tfrac{1}{2},\\
j, & j-\tfrac{1}{2}\le v<j+\tfrac{1}{2}, \quad j=1,\dots,K,\\
K+1, & v\ge K+\tfrac{1}{2}.
\end{cases}
\end{equation}
% The explicit form of \(Q_{\text{shift}}\) is
% \begin{equation}
% Q_{\text{shift}}(v)=
% \begin{cases}
% 0, & 0\le v<2^{-K},\\
% 2^{-k}, & 2^{-k}\le v<2^{-k+1}, \quad k=1,\dots,K,\\
% 1, & v\ge 1.
% \end{cases}
% \end{equation}

\subsubsection{Quantization Error Analysis}

Assuming \(\mathbb{E}[X]<\infty\), we first analyze the expected absolute quantization error
\begin{equation}
\mathcal{E}_A=\mathbb{E}\!\left[\left\lvert X-Q_A(X)\right\rvert\right], \qquad A\in\{\text{shift},\text{int}\}.
\end{equation}
For the sensing applications considered in this paper, the membrane distributions are not naturally uniform on the original linear scale. Instead, as illustrated by the empirical measurements in Figure~\ref{fig:quant_analysis}, they are concentrated near zero while remaining spread across multiple logarithmic scales. In this regime, the dyadic grid of ShiftLIF is a more natural match to the underlying distribution, because it allocates finer resolution to the low-value region where most membrane mass is concentrated while still preserving multiple spike levels across scales. The lemma below gives a simple sufficient condition under which this allocation yields a smaller expected absolute error than INT-LIF.

\begin{lemma}
The expected absolute quantization error \(\mathcal{E}_{\text{shift}}\) is strictly less than \(\mathcal{E}_{\text{int}}\) if
\begin{equation}
2^{-K}\Pr(2^{-K}\le X<\tfrac{1}{2})
>
\tfrac{1}{2}\Pr(\tfrac{3}{4}\le X<1)
+K\Pr(X\ge \tfrac{3}{2}).
\end{equation}
\end{lemma}

% {\color{purple} todo: visualization of quantization error in real data}

% {\color{gray}

% The proof can be moved to supplementary materials.

\begin{proof}
Define
\begin{equation}
\Delta(v)=\left\lvert v-Q_{\text{int}}(v)\right\rvert-\left(v-Q_{\text{shift}}(v)\right), \quad v\ge 0.
\end{equation}
Since \(Q_{\text{shift}}(v)\le v\) for all \(v\ge 0\), we have
\begin{equation}
\left\lvert v-Q_{\text{shift}}(v)\right\rvert=v-Q_{\text{shift}}(v).
\end{equation}
A direct interval-by-interval comparison of the two quantizers yields the pointwise lower bound
\begin{equation}
\Delta(v)\ge 2^{-K}\mathbf{1}_{\{2^{-K}\le v<\tfrac{1}{2}\}}-\tfrac{1}{2}\mathbf{1}_{\{\tfrac{3}{4}\le v<1\}}-K\mathbf{1}_{\{v\ge \tfrac{3}{2}\}},
\quad v\ge 0.
\end{equation}
Indeed, on \([2^{-K},\tfrac{1}{2})\), one has \(\Delta(v)\ge 2^{-K}\). On \([\tfrac{3}{4},1)\), one has \(\Delta(v)\ge -\tfrac{1}{2}\). On \([\tfrac{3}{2},\infty)\), all cases satisfy \(\Delta(v)\ge -K\). On the remaining region, the right-hand side is nonpositive while the inequality is immediate. Taking expectations yields
\begin{equation}
\begin{aligned}
\mathbb{E}[\Delta(X)&]
\ge
2^{-K}\Pr(2^{-K}\le X<\tfrac{1}{2}) \\
&-\tfrac{1}{2}\Pr(\tfrac{3}{4}\le X<1)
-K\Pr(X\ge \tfrac{3}{2}).
\end{aligned}
\end{equation}
Since
\begin{equation}
\mathcal{E}_{\text{shift}}-\mathcal{E}_{\text{int}}=-\mathbb{E}[\Delta(X)],
\end{equation}
the stated condition implies \(\mathbb{E}[\Delta(X)]>0\), which is equivalent to \(\mathcal{E}_{\text{shift}}<\mathcal{E}_{\text{int}}\).
\end{proof}

% }

% The proof is provided in the supplementary materials.
A direct comparison of the two quantizers shows that ShiftLIF is favored on \([2^{-K},\tfrac{3}{4})\), the two quantizers are tied on \([0,2^{-K})\cup[1,\tfrac{3}{2})\), and INT-LIF is favored on \([\tfrac{3}{4},1)\cup[\tfrac{3}{2},\infty)\). Consequently, the distribution of probability mass within the unit interval matters, not just the aggregate mass \(\Pr(2^{-K}\le X<1)\).
The advantage in quantization error of ShiftLIF on real sensing datasets is evident in Figure~\ref{fig:quant_analysis}.

\subsubsection{Information-Theoretic Capacity}

We now analyze how effectively the two quantizers use the same \(K+2\) output levels. We measure this through the normalized output entropy relative to the bit budget required to encode these levels:
\begin{equation}
U_K^{(A)}=\frac{H(Q_A(X))}{\lceil \log_2(K+2)\rceil}, \qquad A\in\{\text{shift},\text{int}\}.
\end{equation}
Here, \(\lceil \log_2(K+2)\rceil\) is the number of bits needed to encode the \(K+2\) quantization levels. Thus, \(U_K^{(A)}\) measures how effectively those bits are utilized: larger values, especially those closer to \(1\), indicate better use of the available coding capacity and are therefore preferred. For a discrete random variable \(Y\) with probability mass function \(p_Y\), we define
\begin{equation}
H(Y)=-\sum_y p_Y(y)\log_2 p_Y(y),
\end{equation}
with the convention \(0\log_2 0:=0\). For a finite probability vector such as \(R\), \(V\), or \(T\), \(H(R)\), \(H(V)\), and \(H(T)\) are defined by the same formula. For the entropy comparison, the key split occurs at \(\tfrac{1}{2}\): ShiftLIF preserves variation in the low-amplitude region by distributing the mass below \(\tfrac{1}{2}\) across dyadic shells, which can increase output entropy when membrane values are concentrated near zero. By contrast, INT-LIF collapses all mass below \(\tfrac{1}{2}\) to the single output \(0\), and only distributes mass across uniform cells above \(\tfrac{1}{2}\). The next lemma makes this entropy decomposition explicit.

\begin{lemma}
Let \(r=\Pr(X<\tfrac{1}{2})\). Let \(R\), \(V\), and \(T\) denote the conditional output distributions of ShiftLIF below \(\tfrac{1}{2}\), of ShiftLIF on \([\tfrac{1}{2},\infty)\), and of rounded INT-LIF on \([\tfrac{1}{2},\infty)\), respectively. Then
\begin{equation}
\begin{aligned}
H(Q_{\text{shift}}(X)) &= h_2(r)+rH(R)+(1-r)H(V),\\
H(Q_{\text{int}}(X)) &= h_2(r)+(1-r)H(T),
\end{aligned}
\end{equation}
where \(h_2(r)=-r\log_2 r-(1-r)\log_2(1-r)\) is the binary entropy. Consequently,
\begin{equation}
U_K^{(\text{shift})}>U_K^{(\text{int})}
\iff
rH(R)+(1-r)H(V)>(1-r)H(T).
\end{equation}
The same formulas remain valid in the degenerate cases \(r=0\) and \(r=1\), with the absent conditional-entropy term interpreted as zero.
\end{lemma}

The proof is provided in the supplementary materials.
This criterion shows that ShiftLIF is favored when a substantial fraction of the mass lies below \(\tfrac{1}{2}\) and is well spread across the dyadic shells there. INT-LIF is only favored when the mass above \(\tfrac{1}{2}\) is broadly spread across the nearest-integer cells. 
In Figure~\ref{fig:quant_analysis}, we see that ShiftLIF behaves well in terms of bit utilization in sensing datasets.

\subsection{Training Procedure}
\paragraph{Surrogate Gradients}
The shift-quantization operator \(Q_{\mathrm{shift}}\) is piecewise constant and therefore has zero derivative almost everywhere, which prevents direct backpropagation. To enable end-to-end training, we adopt a straight-through estimator (STE) that is matched to the support of the ShiftLIF quantizer.

In the forward pass, the spike output is generated by
\begin{equation}
s_t = Q_{\mathrm{shift}}(u_t).
\end{equation}
Since \(Q_{\mathrm{shift}}\) is defined on the bounded interval \([0,1]\), we approximate its derivative in the backward pass by
\begin{equation}
\frac{\partial s_t}{\partial u_t}
\approx
\mathbf{1}_{[0 \le u_t \le 1]}.
\end{equation}
Accordingly, the gradient of the loss is propagated as
\begin{equation}
\frac{\partial \mathcal{L}}{\partial u_t}
\approx
\frac{\partial \mathcal{L}}{\partial s_t}
\cdot
\mathbf{1}_{[0 \le u_t \le 1]}.
\end{equation}
That is, gradients are passed through unchanged within the quantization support and set to zero outside it.

This choice is consistent with the STE framework widely used in quantized and spiking neural networks, while being specifically adapted to the multi-level support of ShiftLIF. In contrast to binary surrogate functions that primarily provide a gradient around a single firing threshold, the proposed STE supplies a non-zero gradient over the full activation range of the quantizer, allowing all active quantization regions to participate in optimization.

\begin{table*}[t]
\centering
\caption{Comparison of different spiking neuron models across low-filtering wireless sensing, acoustic sensing, motion sensing, and neuromorphic sensing benchmarks. Results are reported as accuracy (\%). Best and second-best results are highlighted in bold and underlined, respectively.}
\label{tab:main_results}
\small
\begin{tabular*}{\textwidth}{@{\extracolsep{\fill}}lccccccccccc@{}}
\toprule
\multirow{2}{*}{Method}
& \multicolumn{4}{c}{Low-filtering Wireless Sensing}
& \multicolumn{2}{c}{Acoustic Sensing}
& \multicolumn{2}{c}{Motion Sensing}
& \multicolumn{2}{c}{Neuromorphic Sensing}
& \multirow{2}{*}{Avg.} \\
\cmidrule(lr){2-5}
\cmidrule(lr){6-7}
\cmidrule(lr){8-9}
\cmidrule(lr){10-11}
& ARIL & UT-HAR & Fi-Humanid & BullyDetect
& UrbanSound8K & GSC
& UCIHAR & HHAR
& CIFAR10-DVS & DVS-Gesture & \\
\midrule
CLIF
& \underline{92.09} & \underline{99.20} & \underline{99.45} & 81.38
& \textbf{83.06} & 89.54
& \underline{93.86} & \underline{91.77}
& \underline{63.65} & \underline{93.40}
& \underline{88.74} \\

GLIF
& 89.93 & 98.80 & 98.17 & 78.26
& 79.01 & 90.31
& 92.74 & 91.28
& 63.15 & \underline{93.40}
& 87.50 \\

INTLIF
& 89.57 & \textbf{99.40} & 98.17 & 79.69
& 79.67 & 90.51
& 93.28 & 91.68
& 62.15 & \textbf{95.14}
& 87.93 \\

LIF
& 91.01 & 97.09 & 99.08 & 67.28
& 79.37 & 84.90
& 83.71 & 81.02
& 46.65 & 89.93
& 82.00 \\

PLIF
& 91.73 & 98.90 & 98.72 & 79.91
& 82.76 & 90.39
& 93.25 & 91.71
& \textbf{64.35} & 93.06
& 88.48 \\

PSN
& 85.61 & 96.89 & 99.27 & 63.39
& 76.29 & 89.31
& 92.91 & 90.50
& 56.65 & 91.32
& 84.21 \\

TLIF
& 88.13 & 96.39 & \underline{99.45} & 15.63
& 78.52 & 52.39
& 91.52 & 69.80
& 22.85 & 84.03
& 69.87 \\

ILIF
& 91.37 & \textbf{99.40} & 99.08 & \textbf{82.37}
& 81.97 & \underline{90.70}
& 93.69 & 91.36
& \textbf{64.35} & 92.01
& 88.63 \\

RPLIF
& 87.77 & 97.39 & 99.27 & 67.68
& 80.94 & 85.59
& 91.25 & 81.71
& 45.05 & 88.19
& 82.48 \\

\midrule
\textbf{ShiftLIF}
& \textbf{92.81} & \textbf{99.40} & \textbf{99.63} & \underline{82.14}
& \underline{83.00} & \textbf{91.31}
& \textbf{94.30} & \textbf{94.75}
& \underline{63.65} & 92.36
& \textbf{89.34} \\
\bottomrule
\end{tabular*}
\end{table*}

\paragraph{Spike rate regularization.}
To further control the firing activity during training, we add a spike rate regularizer to the task loss:
\begin{equation}
\mathcal{L}=\mathcal{L}_{\mathrm{CE}}+\lambda_{\mathrm{sr}}\mathcal{L}_{\mathrm{sr}},
\end{equation}
where
\begin{equation}
\mathcal{L}_{\mathrm{sr}}=\frac{1}{L}\sum_{\ell=1}^{L}\left[\mathrm{mean}\!\left(|\mathcal{S}^{(\ell)}|\right)-r^*\right]^+.
\end{equation}
This term penalizes only layers whose average spike magnitude exceeds the target \(r^*\), thereby reducing excessive firing activity without suppressing already sparse responses. We use an L1-style penalty so that the regularization signal remains stable near the target rate.

\subsection{Hardware Computation}
A key motivation for adopting a power-of-two spike alphabet is to avoid the general multiplication overhead introduced by conventional multi-level spiking neurons. In standard synaptic computation, the pre-synaptic input to neuron \(i\) at time step \(t\) is given by
\begin{equation}
x_i^{(t)} = \sum_j W_{ij} s_j^{(t)},
\end{equation}
where \(W_{ij}\) denotes the synaptic weight and \(s_j^{(t)}\) is the pre-synaptic spike. For conventional multi-level SNNs, \(s_j^{(t)}\) typically takes arbitrary integer or floating-point values, so evaluating \(W_{ij}s_j^{(t)}\) generally requires standard multiplication.

In ShiftLIF, all non-zero spikes are constrained to the form
\begin{equation}
s_j^{(t)} = 2^{-k}, \qquad k \in \{0,1,\dots,K\}.
\end{equation}
Under integer or fixed-point weight representations, the corresponding synaptic product can be implemented as a right-shift operation,
\begin{equation}
W_{ij} s_j^{(t)} = W_{ij} \cdot 2^{-k} \;\;\longrightarrow\;\; W_{ij} \gg k,
\end{equation}
where \(\gg\) denotes arithmetic right shift. When \(s_j^{(t)}=0\), the synaptic update is skipped entirely, preserving the event-driven sparsity of spiking computation.

As a result, ShiftLIF replaces MAC(multiply-accumulate) operations in multi-level synaptic interactions with shift-based accumulation. This preserves the hardware efficiency of standard binary SNN computation while enabling richer spike representations. In addition, the reduced effective bit width of shifted operands can make the subsequent accumulation more efficient in practical integer implementations. Compared with conventional multi-level spike designs, ShiftLIF therefore provides a more favorable trade-off between representational capacity and hardware efficiency.

\begin{table*}[t]
\centering
\caption{Comparison between ShiftLIF and INT-LIF across different backbones on wireless sensing benchmarks. For each dataset, we report the accuracy (\%) of ShiftLIF, the accuracy (\%) of INT-LIF, and their difference $\Delta$ = ShiftLIF $-$ INT-LIF. The better result between the two neurons is highlighted in bold.}
\label{tab:1}
\setlength{\tabcolsep}{4pt}
\small
\resizebox{\textwidth}{!}{
\begin{tabular}{lccc ccc ccc ccc}
\toprule
\multirow{2}{*}{Backbone}
& \multicolumn{3}{c}{ARIL}
& \multicolumn{3}{c}{UT-HAR}
& \multicolumn{3}{c}{Fi-Humanid}
& \multicolumn{3}{c}{BullyDetect} \\
\cmidrule(lr){2-4} \cmidrule(lr){5-7} \cmidrule(lr){8-10} \cmidrule(lr){11-13}
& ShiftLIF & INT-LIF & $\Delta$
& ShiftLIF & INT-LIF & $\Delta$
& ShiftLIF & INT-LIF & $\Delta$
& ShiftLIF & INT-LIF & $\Delta$ \\
\midrule
SpikingLeNet      & \textbf{92.01} & 89.57 & +2.44 & \textbf{99.40} & \textbf{99.40} & 0.00 & \textbf{99.63} & 98.17 & +1.46 & \textbf{82.14} & 79.69 & +2.45 \\
SpikingVGG9       & \textbf{97.84} & 94.60 & +3.24 & \textbf{99.90} & 99.60 & +0.30 & \textbf{99.63} & 99.45 & +0.18 & \textbf{80.36} & 30.62 & +49.74 \\
SEWResNet34       & 89.21 & \textbf{89.57} & -0.36 & \textbf{99.80} & 99.30 & +0.50 & \textbf{99.45} & \textbf{99.45} & 0.00 & \textbf{79.55} & 58.53 & +21.02 \\
MSResNet34        & \textbf{93.17} & 88.85 & +4.32 & \textbf{99.60} & 99.40 & +0.20 & \textbf{99.45} & 99.08 & +0.37 & \textbf{80.13} & 77.86 & +2.27 \\
Spikformer256     & \textbf{94.96} & 89.57 & +5.39 & \textbf{99.70} & 98.80 & +0.90 & 99.63 & \textbf{99.82} & -0.19 & \textbf{85.94} & 55.54 & +30.40 \\
MetaSpikformer256 & \textbf{94.24} & 92.45 & +1.79 & \textbf{99.80} & 99.70 & +0.10 & \textbf{99.63} & 99.08 & +0.55 & \textbf{77.59} & 67.99 & +9.60 \\
QKformer256       & \textbf{96.04} & 89.21 & +6.83 & \textbf{99.90} & 99.10 & +0.80 & 99.63 & \textbf{99.82} & -0.19 & \textbf{74.82} & 59.51 & +15.31 \\
\bottomrule
\end{tabular}
}
\end{table*}

\section{Experiment}
\subsection{Setup}
\subsubsection{Dataset}
To evaluate ShiftLIF, we benchmark across 10 datasets with four distinct sensing modalities: (1) \textbf{Wireless Sensing:} ARIL~\cite{wang2019joint}, UT-HAR~\cite{yang2023sensefi}, Fi-HumanID~\cite{yang2023sensefi}, and BullyDetect~\cite{lan2024bullydetect} for WiFi-based human activity and identification tasks; (2) \textbf{Acoustic Sensing:} UrbanSound8K~\cite{salamon2014dataset} (10-fold cross-validation) and Google Speech Commands v2 (GSC, 35-class)~\cite{warden2018speech}; (3) \textbf{Motion Sensing:} UCI-HAR~\cite{anguita2013public} and HHAR~\cite{stisen2015smart} for IMU-based activity recognition; and (4) \textbf{Neuromorphic Vision:} CIFAR10-DVS~\cite{li2017cifar10} and DVS128-Gesture~\cite{amir2017low}.
These datasets include continuous signals that require initial encoding as well as data from native event sensors. This variety allows us to evaluate the proposed neuronal dynamics under different spatial and temporal distributions. Continuous motion signals are segmented using a sliding window with a size of 128 and a stride of 64, while acoustic waveforms are transformed into 128-dimensional Mel-spectrograms.

\subsubsection{Besealine Neurons}
We compare ShiftLIF against nine spiking neuron models: the standard binary 
LIF \cite{neftci2019surrogate},
PLIF \cite{fang2021incorporating}, 
GLIF \cite{yao2022glif},
CLIF \cite{huang2024clif},
PSN \cite{fang2023parallel},
TLIF \cite{guo2024ternary},
I-LIF \cite{sun2025ilif},
RPLIF \cite{li2025incorporating},
and INT-LIF \cite{luo2024integer}.
Among these, INT-LIF is the most direct comparator as it also provides multi-level spike outputs, while the remaining baselines emit binary $\{0,1\}$ or ternary $\{-1,0,1\}$ spikes.

\subsubsection{Implementation} The comparison of different neurons is done with an identical spiking LeNet~\cite{lecun2002gradient} backbone. Models are trained from scratch for 150 epochs using the Adam optimizer with the learning rate of $10^{-3}$ and a cosine annealing schedule. The temporal simulation window is fixed to $T=4$ time steps across all tasks. The variants of spike neurons are implemented with spikingjelly~\cite{fang2023spikingjelly}.

For ShiftLIF, we set the precision factor to $K=2$, defining a 2-level spike alphabet $\mathcal{S} = \{0, 2^{-2}, 2^{-1}, 1\}$. To ensure a perfectly fair comparison, the fundamental dynamic parameters are consistently aligned across all evaluated neuron models: the firing threshold is $V_{\text{th}}=1.0$, the reset potential is $V_{\text{reset}}=0.0$, and the membrane time constant is $\tau=2.0$. All experiments are executed on a single NVIDIA GPU with a fixed random seed.

\subsection{Comparison Among Spiking Neurons}

Table~\ref{tab:main_results} compares ShiftLIF with standard LIF, dynamic enhanced variants, and existing multi-level spiking neurons on ten sensing benchmarks. ShiftLIF achieves the best average accuracy, reaching \(89.34\%\), surpassing the \(88.74\%\) and \(88.63\%\) of the strongest competing baselines, CLIF and I-LIF, and improving substantially over \(82.00\%\) of the standard LIF.

Its advantage is most evident on continuous sensing tasks. Across the eight benchmarks in wireless, acoustic, and motion sensing, ShiftLIF ranks first on ARIL, Fi-Humanid, GSC, UCIHAR, and HHAR, and ties for first on UT-HAR. It also achieves the second-best result on BullyDetect and UrbanSound8K. On neuromorphic sensing benchmarks, it is less dominant but remains competitive, reaching \(63.65\%\) on CIFAR10-DVS and \(92.36\%\) on DVS-Gesture. Overall, these results show that ShiftLIF is particularly effective on continuous sensing data while maintaining competitive performance on the other event-based tasks.

\subsection{Comparison Across Backbones}
We further compare ShiftLIF and INT-LIF across a diverse set of backbones, including CNN, ResNet, and Transformer-style architectures~\cite{simonyan2014very, fang2021deep, hu2024advancing, zhou2022spikformer, zhou2024qkformer}. As shown in Table~\ref{tab:1}, ShiftLIF outperforms INT-LIF on the majority of backbone--dataset pairs. The improvement is especially consistent on ARIL, UT-HAR, and BullyDetect, where ShiftLIF is better for nearly all available backbones. On Fi-Humanid, the two neurons are close, but ShiftLIF still achieves equal or higher accuracy in most cases. Only a small number of cases show no gain or a slight decrease, such as SEWResNet34 on ARIL and QKformer256 on Fi-Humanid.

Overall, these results indicate that the advantage of ShiftLIF is not tied to a particular architecture. Instead, it transfers across different backbone families, suggesting that the proposed spike representation is broadly compatible with existing SNN designs.

\subsection{Energy Analysis}
The computational cost of an SNN depends on both the temporal window length $T$ and the average spike rate $s$. A larger $T$ repeats synaptic processing over more timesteps, while a larger $s$ increases the number of effective events that trigger computation and data movement. The energy of a spiking layer is modeled as $E_{\text{total}} = T \cdot s \cdot \left(E_{\text{ACC}} + E_{\text{move}} + E_{\text{weight}}\right)$,
where $E_{\text{ACC}}$ denotes the accumulation energy, $E_{\text{move}}$ denotes the spike movement energy, and $E_{\text{weight}}$ denotes the SRAM weight access energy. All reported energy values are derived from measurements on a commercial 22\,nm process~\cite{yan2024reconsidering}.

Since the temporal window is fixed to $T=4$ for all tasks, the energy differences in Table~\ref{tab:spike_rate} are driven mainly by spike activity and synaptic efficiency. Training with spike rate regularization, ShiftLIF consistently achieves a stronger accuracy--energy trade-off than INT-LIF on the four wireless sensing benchmarks. At the same time, ShiftLIF delivers better accuracy on three of the four datasets. Compared with standard LIF, ShiftLIF also improves accuracy on three datasets while maintaining similar or lower energy in most cases.

\begin{table*}[t]
\centering
\caption{Ablation study on quantization strategy. We compare the proposed logarithmic power-of-two quantization (ShiftLIF) with a uniform multi-level quantization under the same number of spike levels. Results are reported as accuracy (\%). The better result in each column is highlighted in bold.}
\label{tab:3}
\small
\resizebox{\textwidth}{!}{%
\begin{tabular}{lccccccccccc}
\toprule
\multirow{2}{*}{Method}
& \multicolumn{4}{c}{Wireless Sensing}
& \multicolumn{2}{c}{Acoustic Sensing}
& \multicolumn{2}{c}{Motion Sensing}
& \multicolumn{2}{c}{Neuromorphic Sensing}
& \multirow{2}{*}{Avg.} \\
\cmidrule(lr){2-5} \cmidrule(lr){6-7} \cmidrule(lr){8-9} \cmidrule(lr){10-11}
& ARIL & UT-HAR & Fi-HumanID & BullyDetect
& UrbanSound8K & GSC
& UCIHAR & HHAR
& CIFAR10-DVS & DVS-Gesture
&  \\
\midrule
ShiftLIF
& \textbf{92.81} & 99.40 & \textbf{99.63} & \textbf{82.14}
& \textbf{83.00} & 91.31
& 94.30 & \textbf{94.75}
& 63.65 & 92.36
& \textbf{89.34} \\

Uniform
& 89.21 & \textbf{99.50} & 99.27 & 81.25
& 82.21 & \textbf{91.63}
& \textbf{94.37} & 93.51
& \textbf{64.40} & \textbf{95.14}
& 89.05 \\

$\Delta$ (ShiftLIF $-$ Uniform)
& +3.60 & -0.10 & +0.36 & +0.89
& +0.79 & -0.32
& -0.07 & +1.24
& -0.75 & -2.78
& +0.29 \\
\bottomrule
\end{tabular}%
}
\end{table*}

\begin{table}[htbp]
  \centering
  \caption{Spike Rate and Accuracy Comparison Across Datasets and Neuron Types. ShiftLIF* indicates training with spike rate regularization.}                                \label{tab:spike_rate}                                     \begin{tabular}{llccc}
  \toprule            
  Dataset & Neuron & Acc (\%) & Spike Rate & Energy(mJ) \\  \midrule                                                 \multirow{4}{*}{UTHar} 
    & ShiftLIF*  & 98.49 & 0.030 &  3.32  \\
    & LIF             & 97.09 & 0.040 &   4.27   \\ 
    & INTLIF          & 99.40 & 0.046 & 5.44  \\            \midrule   

    \multirow{4}{*}{FiHumanID}                
    & ShiftLIF*   & 98.35 & 0.021 &  2.51  \\ 
    & LIF             & 99.08 & 0.109 &  10.58  \\
    & INTLIF          & 98.17 & 0.203 & 25.71 \\            \midrule                                                              
  \multirow{4}{*}{BullyDetect}         
    & ShiftLIF*   & 82.54 & 0.038 & 4.04 \\
    & LIF             & 67.28 & 0.027 &  3.08  \\ 
    & INTLIF          & 79.69 & 0.080 &  7.97  \\  
    \midrule                                                             
  \multirow{4}{*}{ARIL}             
  & ShiftLIF*  & 92.45 & 0.049 & 5.03  \\
  & LIF       & 91.01 & 0.047 &  4.91  \\ 
  & INTLIF    & 89.57 & 0.151 &  15.19  \\            \bottomrule                                               \end{tabular}                                          \end{table}   

\subsection{Ablation Study}

\paragraph{Effect of precision factor}
We first study the effect of the precision factor $K$, which determines the smallest non-zero spike magnitude $2^{-K}$ and thus the granularity of the spike alphabet. As shown in Figure~\ref{fig3}, the best performance is achieved at moderate values of $K$ rather than increasing monotonically with finer quantization. Specifically, the optimum appears at $K=2$ on ARIL and $K=3$ on BullyDetect, while the coarsest setting $K=1$ gives the worst result on both datasets. This indicates that the improvement of ShiftLIF does not come from simply increasing the number of spike levels. When $K$ becomes larger, the additional bits bring limited benefit. Based on this observation, we use $K=2$ as the default setting in the main experiments.

\paragraph{Logarithmic vs.\ uniform quantization.}
We next compare the proposed logarithmic power-of-two quantization with a uniform multi-level quantization under the same number of spike levels. Table~\ref{tab:3} shows that ShiftLIF achieves better overall performance, with clearer gains on continuous sensing tasks such as ARIL, Fi-HumanID, BullyDetect, UrbanSound8K, and HHAR. This result is consistent with our design motivation: membrane potentials in trained SNNs are typically concentrated in the low-amplitude regime, where the logarithmic grid provides finer resolution than a uniform one. The advantage is less consistent on event-based vision benchmarks, suggesting that the benefit of quantization depends on the underlying signal statistics. Overall, these results support the central design of ShiftLIF: the gain comes not only from using multi-level spikes, but from using a quantization scheme that is better matched to SNN membrane dynamics.

\begin{figure}[t]
    \centering
    \begin{subfigure}[t]{0.48\columnwidth}
        \centering
        \includegraphics[width=\linewidth]{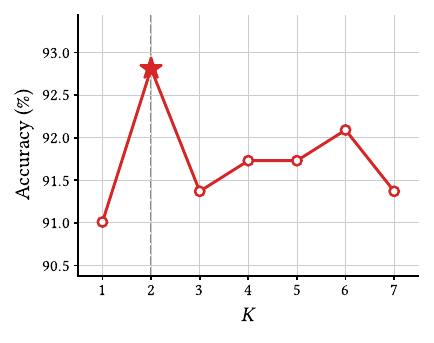}
        \caption{ARIL}
        \label{fig:ablation_k_aril}
    \end{subfigure}
    \hfill
    \begin{subfigure}[t]{0.48\columnwidth}
        \centering
        \includegraphics[width=\linewidth]{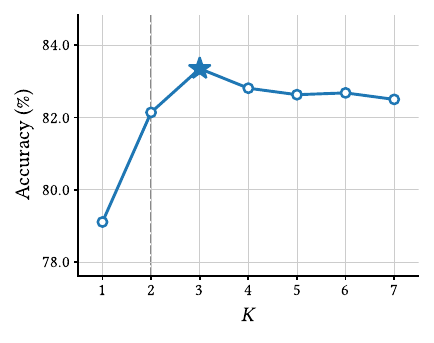}
        \caption{BullyDetect}
        \label{fig:ablation_k_bully}
    \end{subfigure}
    \caption{Ablation on the precision factor $K$ of ShiftLIF on ARIL and BullyDetect. Moderate $K$ values achieve the best accuracy, while increasing $K$ further brings no consistent gain.}
    \label{fig3}
\end{figure}

\section{Discussion}
The results show a clear pattern that ShiftLIF is most effective on continuous sensing tasks, including wireless, acoustic, and motion data, while its advantage is smaller on event-based vision benchmarks. A likely reason is that continuous signals contain useful amplitude variations over time, and these variations are not well preserved by binary spikes. By using a multi-level spike set with finer resolution near small values, ShiftLIF can retain more of this information during neuron-to-neuron communication.

This observation is also consistent with the membrane-potential distributions analyzed in our study. Since a large fraction of membrane states lie in the low-amplitude regime, a logarithmic quantization grid is better matched to the underlying signal statistics than a uniform grid. At the same time, the power-of-two constraint keeps synaptic computation efficient, since spike-weight interactions can still be implemented through shift and accumulation rather than general multiplication. The main benefit of ShiftLIF therefore, comes from combining richer spike representation with a lightweight computation path.

\section{Conclusion}
In this paper, we propose ShiftLIF, a simple multi-level spiking neuron that improves both accuracy and hardware efficiency. By using a logarithmic power-of-two spike set, ShiftLIF better matches the membrane potential distribution and preserves more useful information, especially for small activations. Experiments on ten datasets across four edge-sensing modalities show that ShiftLIF consistently achieves a better accuracy--energy trade-off than existing neuron designs. The gains are especially clear on continuous sensing tasks, where fine signal changes matter. We also show that logarithmic quantization works better than uniform quantization, and that a moderate precision factor $K$ is enough to reach strong performance. Overall, ShiftLIF provides a simple and effective way to improve SNNs.

%In this paper, we propose ShiftLIF, a novel spiking neuron model that resolves the fundamental tension between representational capacity and hardware efficiency in Spiking Neural Networks. By constraining the spike alphabet to a logarithmically spaced, power-of-two grid, ShiftLIF perfectly aligns its quantization resolution with the inherently zero-concentrated distribution of membrane potentials, maximizing information retention in the critical low-amplitude regime. Extensive evaluations across ten datasets and four distinct edge-sensing modalities demonstrate that ShiftLIF consistently delivers a strictly superior accuracy-efficiency trade-off. These representational gains are particularly pronounced in continuous sensing tasks, where preserving fine-grained, sub-threshold signal variations is critical. Furthermore, empirical analyses confirm our theoretical intuitions: logarithmic quantization systematically outperforms uniform alternatives, and a moderate precision factor ($K$) is sufficient to saturate task performance. Ultimately, ShiftLIF establishes that distribution-aware spike discretization is a simple yet highly effective paradigm for advancing neuromorphic computing. Future work will focus on scaling this shift-based mechanism to transformer-level architectures and validating the theoretical energy savings on custom neuromorphic silicon.

\clearpage
\newpage

\bibliographystyle{ACM-Reference-Format}
\bibliography{ref}

\clearpage

\appendix
\section{Theoretical Analysis}

\begingroup
\renewcommand{\thetheorem}{3.2}
\begin{lemma}
Let \(r=\Pr(X<\tfrac{1}{2})\). Let \(R\), \(V\), and \(T\) denote the conditional output distributions of ShiftLIF below \(\tfrac{1}{2}\), of ShiftLIF on \([\tfrac{1}{2},\infty)\), and of rounded INT-LIF on \([\tfrac{1}{2},\infty)\), respectively. Then
\begin{equation}
\begin{aligned}
H(Q_{\text{shift}}(X)) &= h_2(r)+rH(R)+(1-r)H(V),\\
H(Q_{\text{int}}(X)) &= h_2(r)+(1-r)H(T),
\end{aligned}
\end{equation}
where \(h_2(r)=-r\log_2 r-(1-r)\log_2(1-r)\) is the binary entropy. Consequently,
\begin{equation}
U_K^{(\text{shift})}>U_K^{(\text{int})}
\iff
rH(R)+(1-r)H(V)>(1-r)H(T).
\end{equation}
The same formulas remain valid in the degenerate cases \(r=0\) and \(r=1\), with the absent conditional-entropy term interpreted as zero.
\end{lemma}
\endgroup

\begin{proof}
Assume first that \(0<r<1\), and define
\begin{equation}
\begin{aligned}
m &= \Pr(\tfrac{1}{2}\le X<1),\\
c &= \Pr(X\ge 1),
\end{aligned}
\end{equation}
so that \(r+m+c=1\). Let \(R=(R_0,R_2,\dots,R_K)\) be the dyadic-shell distribution below \(\tfrac{1}{2}\), where
\begin{equation}
\begin{aligned}
R_0 &= \Pr(X<2^{-K}\mid X<\tfrac{1}{2}),\\
R_k &= \Pr(2^{-k}\le X<2^{-k+1}\mid X<\tfrac{1}{2}), \quad k=2,\dots,K,
\end{aligned}
\end{equation}
let
\begin{equation}
V=\left(\frac{m}{1-r},\frac{c}{1-r}\right),
\end{equation}
and let \(T=(T_1,T_2,\dots,T_{K+1})\) be the rounded-INT distribution on \([\tfrac{1}{2},\infty)\), where
\begin{equation}
\begin{aligned}
T_j &= \Pr(j-\tfrac{1}{2}\le X<j+\tfrac{1}{2}\mid X\ge \tfrac{1}{2}), \quad j=1,\dots,K,\\
T_{K+1} &= \Pr(X\ge K+\tfrac{1}{2}\mid X\ge \tfrac{1}{2}).
\end{aligned}
\end{equation}
For ShiftLIF, the outputs below \(\tfrac{1}{2}\) have total mass \(r\) and conditional distribution \(R\), while the two outputs on \([\tfrac{1}{2},\infty)\) have total mass \(1-r\) and conditional distribution \(V\). By the chain rule for Shannon entropy,
\begin{equation}
H(Q_{\text{shift}}(X))=h_2(r)+rH(R)+(1-r)H(V).
\end{equation}
For rounded INT-LIF, the output \(0\) corresponds to the single event \(\{X<\tfrac{1}{2}\}\) of probability \(r\), while the remaining \(K+1\) outputs have total mass \(1-r\) and conditional distribution \(T\). Another application of the chain rule gives
\begin{equation}
H(Q_{\text{int}}(X))=h_2(r)+(1-r)H(T).
\end{equation}
Because both utilizations are normalized by the same denominator \(\lceil \log_2(K+2)\rceil\),
\begin{equation}
U_K^{(\text{shift})}>U_K^{(\text{int})}
\iff
H(Q_{\text{shift}}(X))>H(Q_{\text{int}}(X)).
\end{equation}
Substituting the two entropy identities above yields
\begin{equation}
U_K^{(\text{shift})}>U_K^{(\text{int})}
\iff
rH(R)+(1-r)H(V)>(1-r)H(T),
\end{equation}
which is the desired criterion.

It remains to consider the edge cases. If \(r=1\), then \(Q_{\text{int}}(X)\equiv 0\) and \(H(Q_{\text{int}}(X))=0\), while \(Q_{\text{shift}}(X)\) is distributed according to \(R\); this agrees with the stated formula after setting \((1-r)H(V)=(1-r)H(T)=0\). If \(r=0\), then ShiftLIF has only the two outputs \(\tfrac{1}{2}\) and \(1\), distributed according to \(V\), while rounded INT-LIF is distributed according to \(T\); again the same formula holds with \(rH(R)=0\).
\end{proof}

  \begin{algorithm}[t]                                                                                                    
  \caption{SNN Training with Spike Rate Regularization}                                           
  \label{alg:spike-rate-reg}                                                                                              
  \begin{algorithmic}[1]                                                                                                  
  \REQUIRE Training set $\mathcal{D}_{\text{train}}$, pretrained weights $\boldsymbol{\theta}_0$, number of epochs $E$,   
  learning rate $\eta$, spike rate penalty weight $\lambda_{\text{sr}}$, target spike magnitude $r^*$                     
  \ENSURE Trained model parameters $\boldsymbol{\theta}^*$                      
  \STATE Initialize SNN model $f_{\boldsymbol{\theta}}$ with $L$ spiking layers; load $\boldsymbol{\theta} \leftarrow     
  \boldsymbol{\theta}_0$                                                                                                  
  \STATE $\text{acc}_{\text{best}} \leftarrow 0$                                   
  \FOR{$e = 1$ \TO $E$}                                                                                                   
      \FOR{each mini-batch $(\mathbf{X}, \mathbf{y}) \in \mathcal{D}_{\text{train}}$}                                     
          \STATE \textit{// Forward pass over $T$ timesteps}                                                              
          \FOR{$t = 1$ \TO $T$}                                                                                           
              \FOR{$\ell = 1$ \TO $L$}                                                                                    
                  \STATE $\mathbf{v}^{(\ell)}_t \leftarrow \mathbf{v}^{(\ell)}_{t-1} + \left( v_{\text{reset}} -          
  \mathbf{v}^{(\ell)}_{t-1} + \mathbf{z}^{(\ell)}_t \right) / \tau$ \hfill $\triangleright$ \textit{LIF charging}         
                  \STATE $\mathbf{s}^{(\ell)}_t \leftarrow \mathcal{Q}_{\text{shift}}\!\left(\mathbf{v}^{(\ell)}_t\right)$
   \hfill $\triangleright$ \textit{Multi-level spike: $\{0, 2^{k_{\min}}, \ldots, 2^0\}$}                                 
                  \STATE $\mathbf{v}^{(\ell)}_t \leftarrow \mathbf{v}^{(\ell)}_t - \mathbf{s}^{(\ell)}_t \cdot            
  v_{\text{th}}$ \hfill $\triangleright$ \textit{Soft reset}                                                              
                  \STATE Collect $\mathbf{s}^{(\ell)}_t$ into spike buffer $\mathcal{S}^{(\ell)}$                         
              \ENDFOR                                       
          \ENDFOR                                
          \STATE $\hat{\mathbf{y}} \leftarrow f_{\boldsymbol{\theta}}(\mathbf{X})$
          \STATE \textit{// Task loss}          
          \STATE $\mathcal{L}_{\text{CE}} \leftarrow \text{CrossEntropy}(\hat{\mathbf{y}},\; \mathbf{y})$                 
          \STATE \textit{// Spike rate regularization (differentiable L1 penalty)}                                        
          \STATE $\mathcal{L}_{\text{sr}} \leftarrow \frac{1}{L} \displaystyle\sum_{\ell=1}^{L} \left[\,                  
  \text{mean}\!\left(|\mathcal{S}^{(\ell)}|\right) - r^* \,\right]^{+}$                                                   
          \STATE $\mathcal{L} \leftarrow \mathcal{L}_{\text{CE}} + \lambda_{\text{sr}} \cdot \mathcal{L}_{\text{sr}}$     
          \STATE \textit{// Backward with surrogate gradients and update}    
          \STATE $\boldsymbol{\theta} \leftarrow \boldsymbol{\theta} - \eta \, \nabla_{\boldsymbol{\theta}} \mathcal{L}$  
          \STATE Clear all spike buffers $\mathcal{S}^{(\ell)}$                       
      \ENDFOR                                                                  
      \STATE Evaluate on $\mathcal{D}_{\text{test}}$; \textbf{if} $\text{acc}_{\text{test}} > \text{acc}_{\text{best}}$   
  \textbf{then} $\boldsymbol{\theta}^* \leftarrow \boldsymbol{\theta}$, $\text{acc}_{\text{best}} \leftarrow              
  \text{acc}_{\text{test}}$                                                                                               
  \ENDFOR                                    \RETURN $\boldsymbol{\theta}^*$                              
  \end{algorithmic}                         
  \end{algorithm}

\section{Training with Spike Activity Regularization}
\label{sec:training_algorithm}

To control the overall spiking activity during training, we augment the task loss with a spike activity regularization term. 
The training objective is defined as
\begin{equation}
\mathcal{L}
=
\mathcal{L}_{\mathrm{CE}}
+
\lambda_{\mathrm{sr}} \mathcal{L}_{\mathrm{sr}},
\end{equation}
where $\mathcal{L}_{\mathrm{CE}}$ denotes the task loss, $\mathcal{L}_{\mathrm{sr}}$ is the spike activity regularizer, and $\lambda_{\mathrm{sr}}$ controls its strength. 
For the $\ell$-th spiking layer, we collect its output spikes over all timesteps in the current mini-batch into $\mathcal{S}^{(\ell)}$, and define
\begin{equation}
\mathcal{L}_{\mathrm{sr}}
=
\frac{1}{L}
\sum_{\ell=1}^{L}
\left[
\operatorname{mean}\!\left(\left|\mathcal{S}^{(\ell)}\right|\right) - r^*
\right]^+ ,
\end{equation}
where $L$ is the number of spiking layers, $r^*$ is the target spike activity, and $[\cdot]^+ = \max(0,\cdot)$. 
This term penalizes only the amount by which the average spike activity exceeds the target level, thereby encouraging sparse responses without unnecessarily suppressing already low-activity layers. 
Since our neuron emits multi-level spikes, $\operatorname{mean}(|\mathcal{S}^{(\ell)}|)$ measures the average spike magnitude rather than a strictly binary firing count. 

The training process is shown as Algorithm~\ref{alg:spike-rate-reg}, where $\mathcal{Q}_{\text{shift}}$ denotes the ShiftQuant surrogate function that maps membrane potentials to multi-level spikes in $\{0, 2^{k_{\min}}, \ldots, 2^0\}$, $\mathcal{S}^{(\ell)}$ denotes the spike buffer of the $\ell$-th spiking layer that collects layer outputs over all timesteps in the current mini-batch, $L$ is the number of spiking layers, $T$ is the number of simulation timesteps, $r^*$ is the target spike activity, $\lambda_{\mathrm{sr}}$ is the regularization weight, and $[\cdot]^+$ denotes the positive-part operator $\max(0,\cdot)$.

% \section{Code}
% Our code is available at \href{https://anonymous.4open.science/r/ShiftLIF-0282}{https://anonymous.4open.science/r/ShiftLIF-0282}

\end{document}